\documentclass{article}
\usepackage{spconf,amsmath,graphicx,multicol}
\usepackage{amsfonts}
\usepackage{bbm}
\usepackage{cite}
\usepackage{amssymb}
\usepackage{algorithm}
\usepackage{algorithmic}
\usepackage{tabularx}
\usepackage[utf8]{inputenc}
\usepackage[english]{babel}
\usepackage{amsthm}
\usepackage{subcaption}
\usepackage{mathtools}
\usepackage{bbm}
\usepackage{flushend}

\theoremstyle{definition}

\newtheorem{proposition}{Proposition}
\newtheorem{theorem}{Theorem}[section]

\newtheorem{lemma}[theorem]{Lemma}

\DeclareMathOperator{\E}{\mathbb{E}}
\def\SNR{{\mathrm{SNR}}}

\def\y{{\mathbf y}}
\def\e{{\boldsymbol{\eta}}}
\def\x{{\boldsymbol{\beta}}}

\def\H{{\mathbf H}}
\def\C{{\mathbf C}}

\def\P{{\mathbf P}}

\def\R{{\mathbb{R}}}

\def\I{{\mathbf I}}

\def\r{{\mathbf r}}
\def\F{{\boldsymbol{\Phi}}}
\def\Si{{\boldsymbol{\Psi}}}

\def\w{{\mathbf w}}

\def\r{{\mathbf r}}
\def\A{{\mathbf H}}

\def\P{{\mathbf P}}

\def\R{{\mathbb{R}}}

\def\I{{\mathbf I}}
\def\a{{\mathbf a}}
\def\b{{\mathbf b}}
\def\r{{\mathbf r}}
\def\u{{\mathbf u}}
\def\c{{\mathbf z}}

\newcommand{\ts}{\textsuperscript}
\newcommand\norm[1]{\left\lVert#1\right\rVert}
\newcommand{\argmax}{\arg\!\max}
\newcommand{\argmin}{\arg\!\min}

\let\oldref\ref
\renewcommand{\ref}[1]{(\oldref{#1})}
\newcommand{\RNum}[1]{\uppercase\expandafter{\romannumeral #1\relax}}

\makeatletter
\renewcommand{\fnum@figure}{Fig. \thefigure}
\makeatother

\title{Sparse recovery via Orthogonal Least-Squares under presence of Noise}
\name{Abolfazl Hashemi and Haris Vikalo}
\address{Department of Electrical and Computer Engineering\\
		University of Texas at Austin, Austin, TX, USA}
\begin{document}
%
\maketitle
\begin{abstract}
We consider the Orthogonal Least-Squares (OLS) algorithm for the recovery of a $m$-dimensional $k$-sparse signal from a low number of noisy linear measurements. 
The Exact Recovery Condition (ERC) in bounded noisy scenario is established for OLS under certain condition on nonzero elements of the signal. The new result also improves the existing guarantees for Orthogonal Matching Pursuit (OMP) algorithm. In addition, This framework is employed to provide probabilistic guarantees for the case that the coefficient matrix is drawn at random according to Gaussian or Bernoulli distribution where we exploit some concentration properties. It is shown that under certain conditions, OLS recovers the true support in $k$ iterations with high probability. This in turn demonstrates that ${\cal O}\left(k\log m\right)$ measurements is sufficient for exact recovery of sparse signals via OLS.
\end{abstract}
\begin{keywords}
compressed sensing, orthogonal least-squares, sparse reconstruction, exact recovery condition, orthogonal matching pursuit
\end{keywords}
\section{Introduction}
\label{sec:intro}
In many practical scenarios signal of interest can be modeled  as  a  sparse  solution  to  an  underdetermined  linear  systems of equations. Examples include sparse linear regression \cite{tipping2001sparse}, compressed sensing \cite{donoho2006compressed}, sparse channel
estimation in communication systems \cite{carbonelli2007sparse,barik2014sparsity}, compressive DNA microarrays \cite{parvaresh2008recovering} and a number of other applications in signal processing and machine learning \cite{lustig2007sparse,elad2010role,elhamifar2009sparse}. Consider the linear 
measurement model
\begin{equation} \label{eq:1}
{\y=\A\x+\e},
\end{equation}
where $\y \in\R^{n}$ denotes the vector of observations, $\A \in\R^{n \times m}$ is the coefficient matrix (i.e., a collection of features) assumed to be full rank, ${\bf e} \in\R^{n}$ is the additive observation noise vector, and $\x \in\R^{m}$ is a vector known to have at most $k$ non-zero components (i.e., $k$ is the sparsity level of $\x$).  Finding a sparse approximation to $\x$ leads to a cardinality-constraint optimization problem. In particular, we would like to solve the so-called $l_0$-constrained least-squares
\begin{equation}  \label{eq:2}
\begin{aligned}
& \underset{\x}{\text{minimize}}
&  \norm{\y-\A\x}^{2}_{2}
& &\text{subject to}
& & \norm{\x}_{0} \leq k.
\end{aligned}
\end{equation}
The number of possible locations of non-zero entries in $\x$ scales combinatorially with $n$ which renders
\ref{eq:2} computationally challenging; in fact, the problem is NP-hard. 

Accelerated
approaches to find an approximate solution 
include a number of iterative heuristics that attempt to solve \ref{eq:2} greedily by identifying columns of $\A$ which correspond to non-zero components of $\x$ via locally optimal decisions. 
Among the greedy methods, orthogonal matching pursuit (OMP) algorithm \cite{pati1993orthogonal} has attracted particular attention in recent years. 
There have been 
numerous modifications of OMP proposed in the literature that enhance the performance of OMP \cite{donoho2012sparse,dai2009subspace,wang2012generalized}. The principal idea in these methods is to select multiple ``good" indices in each iteration in order to reduce the cost of identification step and recover the true support in fewer iterations.  Performance of OMP is evaluated in numerous scenarios and necessary and sufficient conditions for exact reconstruction are provided\cite{mo2012remark,tropp2004greed,cai2011orthogonal,das2011submodular,tropp2007signal,fletcher2012orthogonal}.

Orthogonal Least-Squares (OLS) method \cite{chen1989orthogonal} has drawn attention in recent years \cite{soussen2013joint,maymon2015viterbi,a2bol2016} and its provable performance is analyzed in some scenarios. In \cite{soussen2013joint} OLS is analyzed under the Exact Recovery Condition (ERC) for noiseless setup. Herzet et al. \cite{herzet2016relaxed} provided coherence-based conditions for sparse recovery of signals via OLS when nonzero components obey some decay. In \cite{herzet2013exact}, sufficient conditions for exact recovery from noise-free data is stated when a subset of optimal indices is available. However, the existing analysis and performance guarantees for OLS are predominantly limited  to the case of non-random measurements in noise-free setting.

In this paper, we establish ERC-based guarantees for OLS and show that given certain SNR conditions, OLS recovers location of nonzero elements of $\x$ in first $k$ iteration. Furthermore, a probabilistic result is provided for the case of random measurements with $l_2$-bounded additive noise where coefficient matrix $\A$ is drawn at random from Gaussian or Bernoulli distribution. Specifically, we find a lower bound on the probability of sparse recovery in  $k$ iterations when the nonzero element of unknown vector with smallest magnitude satisfies certain condition. Consequently, we demonstrate that with ${\cal O}\left(k\log m\right)$ measurements OLS will succeed with probability arbitrarily close to one.

The rest of the paper is organized as follows. Section 2 reviews the preliminaries. In section 3, we establish sufficient conditions on exact support recovery from perturbed measurements. Section 4 provides our results on performance
guarantee of OLS from linear random
measurements.  Some concluding remarks are provided in Section 5.
\section{Preliminaries and Notations} \label{sec:ols}
In this paper, Bold capital letters refer to matrices and bold lowercase letters represent vectors. For matrix $\A \in \R^{n\times m}$ with full column rank, i.e., $m>n$, $\A_{ij}$ denotes the entry in the $i\ts{th}$ row and column $j\ts{th}$, $\a_j$ refers to  the $j\ts{th}$ 
column of $\A$, and $\A_k \in \R^{n\times k}$ is one of the ${m}\choose{k}$ submatrices of $\A$. ${\cal L}_\A$ is the subspace spanned by columns of $\A$. ${\P}_\A^{\bot}=\I-\A \A^\dagger$ is the orthogonal projection operator onto orthogonal complement of ${\cal L}_\A$ where
$\A^\dagger=\left(\A^{\top}\A\right)^{-1}\A^{\top}$ is the Moore-Penrose pseudo-inverse of $\A$ and $\I$ is the identity matrix whose dimension is equal to the number of rows in $\A$. Similar notations are defined for ${\cal L}_k$, $\A_{k}^\dagger$, and ${\P}_k$ where we drop the subscript $\A$ for simplicity.  Specifically, ${\P}_k$ is the projection operator onto ${\cal L}_k$. 

Let ${\cal I}=\{1,\dots,m\}$ be the set of all indices, ${\cal S}_{opt}=\{1,\dots,k\}$ be the set of indices corresponding to nonzero elements of $\x$, and ${\cal S}_i$ be the set of selected indices at the end of $i\ts{th}$ iteration of OLS. For the sets ${\cal T}_1 \subset {\cal I}$ and ${\cal T}_2 \subset {\cal I}$ define
$\b_{j}^{{\cal T}_1}=
\frac{\P_{{\cal T}_1}^\bot\a_j}{\norm{\P_{{\cal T}_1}^\bot \a_j}_2},\hspace{0.2cm}  j\in {\cal T}_2$
where $\P_{{\cal T}_1}^\bot$ is the orthogonal projection matrix onto orthogonal complement of subspace spanned by columns of $\A$ with indices in ${\cal T}_1$. $\mathbbm{1}\{.\}$ is the indicator function and is equal to 1 if its argument holds and zero otherwise.

For a scalar random variable $X$, the notation $X \sim {\cal B}(\frac{1}{2},\pm 1)$ denotes that $X$ is a Bernoulli random variable and takes values $1$ and $-1$ with equal probability. For non-scalar object such as matrix $\A$, $\A \sim {\cal N}\left(0,\frac{1}{n}\right)$ means entries of $\A$ are drawn independently according to a zero-mean Gaussian distribution with variance $\frac{1}{n}$. Similar definition holds for $\A \sim {\cal B}(\frac{1}{2},\pm \frac{1}{\sqrt{n}})$.

We recall the principles of the OLS algorithm. OLS sequentially projects columns of $\A$ onto a residual vector and selects the column that
leads to the smallest residual norm. Specifically, OLS chooses a new index $j_s$ in $i\ts{th}$ iteration by employing the following criterion:
\begin{equation}\label{eq:ools}
{j}_{s}=\argmin_{j \in {\cal I}\backslash {\cal S}_{i-1}}{\norm{\y - \A_{{\cal S}_{i-1}\cup\{j\}}\A_{{\cal S}_{i-1}\cup\{j\}}^{\dagger}\y}_2}
\end{equation}
This procedure is computationally more expensive than OMP since in 
addition to solving a least-square problem to update the residual vector, orthogonal projection of each column 
needs to be found at each step of OLS. 
Note that the performances of OLS and OMP are identical when the columns of $\A$ are orthogonal. \footnote{In fact, orthogonality of the columns of A leads to a modular objective
function in \ref{eq:2}, implying optimality of both methods when $\e=0$.} However, for coherent and Redundant dictionaries OLS outperforms OMP. See \cite{herzet2013exact} for a detailed discussion
It is shown in \cite{hashemi2016sparse} that index selection criterion in  \ref{eq:ools} can alternatively be written as 
\begin{equation}\label{eq:nols}
j_s=\argmax_{j \in {\cal I}\backslash {\cal S}_{i-1}}{\left|\r_{i-1}^{\top}\frac{{\bf P}_{i-1}^{\bot}{\bf a}_{j}}{\norm{{\bf P}_{i-1}^{\bot}{\bf a}_{j}}_2}\right|}
\end{equation}
where $\r_{i-1}$ is the residual vector in $i\ts{th}$ iteration. In addition, projection matrix needed for subsequent iteration is related to the current projection matrix by the following recursion,
\begin{equation}
{\P}_{i+1}^{\bot}={\P}_i^{\bot}-\frac{{\P}_i^{\bot}\a_{j_s}\a_{j_s}^{\top} {\P}_i^{\bot}}{\norm{{\P}_i^{\bot}\a_{j_s}}_2^2}.
\label{eq:perp}
\end{equation}

Before formalizing the main results, we start by some useful lemmas that are employed in the proofs of main theorems.

\begin{lemma}\label{lem:nois} 
\textit{Let $\A_1$, $\A_2$, and $\C$ be full rank tall matrices such that $\C=[\A_1,\A_2]$. Then for $i=1,2$}
\begin{equation} \label{eq:sigA}
\begin{aligned} 
&\sigma_{\min}\left(\A_i\right)\geq \sigma_{\min}\left(\C\right), &\sigma_{\max}\left(\A_i\right)\leq \sigma_{\max}\left(\C\right)
\end{aligned}
\end{equation}
\end{lemma}
\begin{lemma} \label{lem:3}
\textit{The noise term in \ref{eq:1} can equivalently be written as 
\begin{equation}
\e=\bar{\A}\w+\e^\bot
\end{equation}
where $\e^\bot=\P_k^\bot\e$, $\w=\bar{\A}^\dagger\e$, and $\bar{\A}$ is a submatrix with all true columns of $\A$. In addition, for $(i+1)\ts{st}$ iteration
\begin{subequations}
\begin{equation}
\r_i=\e^\bot+\P_i^\bot\bar{\A}_{i^c}\c_{i^c}
\end{equation}
\begin{equation}
\|\r_i\|_2^2=\|\e^\bot\|_2^2+\|\P_i^\bot\bar{\A}_{i^c}\c_{i^c}\|_2^2
\end{equation}
\end{subequations}}
where $\c=\bar{\x}+\w$, $\bar{\x}$ corresponds to nonzero elements of $\x$, and subscript $i^c$ denotes the set of optimal columns that have not been chosen in first $i$ iterations.
\end{lemma}
\begin{lemma} \label{lem:23}
\textit{Assume $\A \sim {\cal N}(0,1/n)$ or $\A \sim {\cal B}(\frac{1}{2},\pm \frac{1}{\sqrt{n}})$
. Let $\A_k\in\R^{n\times k}$ be a submatrix of $\A$. Then, $\forall {\bf u}\in \R^{n}$ statistically independent of $\A_k$ drawn according to $\u \sim {\cal N}(0,1/n)$ or $\u \sim {\cal B}(\frac{1}{2},\pm \frac{1}{\sqrt{n}})$, it holds that
$\E\norm{\P_k\u}_2^2=\frac{k}{n}\E\norm{\u}_2^2$. In addition, let $c_0(\epsilon)=\frac{\epsilon^2}{4}-\frac{\epsilon^3}{6}$. Then,
\begin{equation}
\Pr\{\big|\norm{\P_k\u}_2^2-\frac{k}{n}\E\norm{\u}_2^2\big|\leq\epsilon\frac{k}{n}\E\norm{\u}_2^2\}\geq 1-2e^{-kc_0(\epsilon)}.
\end{equation}
}
\end{lemma}
\section{Exact Recovery Condition for OLS} \label{sec:mip}
The first analysis of OMP is due to Tropp \cite{tropp2004greed} where he provided sufficient conditions for exact recovery of OMP in noise-free setting. Specifically, let $\bar{\A} \in \R^{n\times k}$ be a matrix with columns indexed by ${\cal S}_{opt}$ and $\widetilde{\A}\in \R^{n\times (m-k)}$ correspond to 
columns indexed by ${\cal I}\backslash {\cal S}_{opt}$. Then if $\e=0$ and
\begin{equation}
M_{OMP}=\|\bar{\A}^\dagger \widetilde{\A}\|_{1,1}<1, 
\end{equation}
OMP recovers support of $\x$ exactly in $k$ iterations. This condition is called Exact Recovery Condition (ERC). Similar results can be established for OLS. In particular, the following proposition holds for $\e=0$.
\begin{proposition}
\textit{Let $\F_{{\cal S}_i}=[\b_{j}^{{\cal S}_i}]\in \R^{n\times(k-i)}$, $j\in {\cal S}_{opt}\backslash{\cal S}_i$ and $\Si_{{\cal S}_i}=[\b_{j}^{{\cal S}_i}]\in\R^{n\times(m-k)}$, $j\in {\cal I}\backslash{\cal S}_{opt}$. Suppose OLS identified true columns in first $i$ iterations. If 
\begin{equation}
M_{i+1}=\|\F_{{\cal S}_i}^\dagger\Si_{{\cal S}_i}\|_{1,1}<1,
\end{equation}
OLS chooses a true column in $(i+1)\ts{st}$ iteration.}
\end{proposition}
The condition $M_i<1$ is the ERC of OLS at $i\ts{th}$ iteration. Here we extend this results to the case that the measurements are perturbed with additive noise where $\|\e\|_2\leq \epsilon_\e$. The following theorem summarizes our main results.
\begin{theorem}\label{thm:erc}
\textit{Suppose $\|\e\|_2 \leq \epsilon_\e$ and that OLS has chosen true columns in first $i$ iteration. Then OLS selects an index from ${\cal S}_{opt}$ at $(i+1)\ts{th}$ iteration if $M_{i+1}<1$ and
\begin{equation}\label{eq:conthm1}
\min_{j}{|\bar{\x}_j|}>  \sigma_{\min}(\bar{\A})\epsilon_\e+\frac{\epsilon_\e}{(1-M_{i+1})\sigma_{\min}^2(\bar{\A})}.
\end{equation}
}
\end{theorem}
\begin{proof}
Proof follows an inductive argument. First, without loss of generality, assume columns of $\A$ are normalized and all nonzero components of $\x$ are in first $k$ locations. This implies that $\A$ can be written in form of $\A=\left[\bar{\A},\widetilde{\A}\right]$.
Assume OLS has selected columns from ${\cal S}_{opt}$ in first $i$ iterations. It then follows from \ref{eq:nols} that
\begin{equation}\label{eq:gsratio}
\rho(\r_i)=\frac{\|\Si_{{\cal S}_i}^\top \r_i\|_\infty}{\|\F_{{\cal S}_i}^\top \r_i\|_\infty}<1.
\end{equation}
is a sufficient condition so that OLS selects a true column at next iteration. However,
\begin{equation}
\begin{aligned}
\rho(\r_i)
&\stackrel{(a)}{\leq}\frac{\|\Si_{{\cal S}_i}^\top\e^\bot+\Si_{{\cal S}_i}^\top\P_i^\bot\bar{\A}_{i^c}\c_{i^c}\|_\infty}{\|\F_{{\cal S}_i}^\top\e^\bot+\F_{{\cal S}_i}^\top\P_i^\bot\bar{\A}_{i^c}\c_{i^c}\|_\infty}\\
&\stackrel{(b)}{\leq}\frac{\|\Si_{{\cal S}_i}^\top\e^\bot\|_\infty+\|\Si_{{\cal S}_i}^\top\P_i^\bot\bar{\A}_{i^c}\c_{i^c}\|_\infty}{\|\F_{{\cal S}_i}^\top\e^\bot+\F_{{\cal S}_i}^\top\P_i^\bot\bar{\A}_{i^c}\c_{i^c}\|_\infty}\\
&\stackrel{(c)}{\leq} \frac{\|\Si_{{\cal S}_i}^\top\e^\bot\|_\infty+\|\Si_{{\cal S}_i}^\top\P_i^\bot\bar{\A}_{i^c}\c_{i^c}\|_\infty}{\|\F_{{\cal S}_i}^\top\P_i^\bot\bar{\A}_{i^c}\c_{i^c}\|_\infty}
\end{aligned}
\end{equation}
where ($a$) is by the equivalence definition of $\r_i$ in Lemma \oldref{lem:nois}, ($b$) follows from triangle inequality, and ($c$) is due to the fact that $\e^\bot$ is orthogonal to ${\cal L}_k$. 
Let $\u=\P_i^\bot\bar{\A}_{i^c}\c_{i^c}$. Hence, we may calculate that
\begin{equation}
\begin{aligned}
\frac{\|\Si_{{\cal S}_i}^\top\u\|_\infty}{\|\F_{{\cal S}_i}^\top\u\|_\infty}&=\frac{\|\Si_{{\cal S}_i}^\top(\F_{{\cal S}_i}^\dagger)^\top\F_{{\cal S}_i}^\top\u\|_\infty}{\|\F_{{\cal S}_i}^\top\u\|_\infty}\\
&\leq \|\Si_{{\cal S}_i}^\top(\F_{{\cal S}_i}^\dagger)^\top\|_{\infty,\infty}\\
&=\|\F_{{\cal S}_i}^\dagger\Si_{{\cal S}_i}\|_{1,1}
\end{aligned}
\end{equation}
owing to the fact that $\u$ lies in ${\cal L}_k\backslash {\cal L}_i$, and the relation between $\|\|_{1,1}$ and $\|\|_{\infty,\infty}$. Therefore,  by definition of $M_i$ in Proposition 1
\begin{equation}
\rho(\r_i)\leq M_{i+1}+\frac{\|\Si_{{\cal S}_i}^\top\e^\bot\|_\infty}{\|\F_{{\cal S}_i}^\top\P_i^\bot\bar{\A}_{i^c}\c_{i^c}\|_\infty}
\end{equation}
Now, observe that applying Lemma \oldref{lem:nois} along with the fact that $\P_i^\bot$ is a projection matrix delivers
\begin{equation}\label{eq:lb}
\|\bar{\A}_{i^c}\P_i^\bot\bar{\A}_{i^c}\c_{i^c}\|_2\geq \sigma_{\min}^2(\bar{\A})\|\c_{i^c}\|_2.
\end{equation}
Consequently, noting $\max_{j\in {\cal S}_{opt}\backslash{\cal S}_i}{\|\P_i^\bot\a_j\|_2}=1$, one may continue to obtain
\begin{equation} \label{eq:rhobound}
\begin{aligned}
\rho(\r_i)
&\leq M_{i+1}+\frac{\|\Si_{{\cal S}_i}^\top\e^\bot\|_\infty}{\|\bar{\A}^\top\P_i^\bot\bar{\A}_{i^c}\c_{i^c}\|_\infty}\\
&\stackrel{(a)}{\leq} M_{i+1}+ \frac{\sqrt{k-i}\|\Si_{{\cal S}_i}^\top\e^\bot\|_\infty}{\sigma_{\min}^2(\bar{\A})\|\c_{i^c}\|_2}\\
&\stackrel{(b)}{\leq} M_{i+1}+\frac{\sqrt{k-i}\epsilon_\e}{\sigma_{\min}^2(\bar{\A})\|\c_{i^c}\|_2}
\end{aligned}
\end{equation}
where ($a$) is by \ref{eq:lb} and the fact that $\c_{i^c} \in \R^{k-i}$, and ($b$) follows from $\|\e\|_2\leq \epsilon_\e$ and the fact that columns of $\Si_{{\cal S}_i}$ have unit $l_2$ norm. Define $\x_{\min}=\min_{j}{|\bar{\x}_j|}$ and $\c_{\min}=\min_{j}{|\c_j|}$. It is easy to check $\c_{\min}\geq\x_{\min}-\|\w\|_2$. Hence, one may obtain
\begin{equation} \label{eq:normc}
\begin{aligned}
\|\c_{i^c}\|_2 &\geq \sqrt{k-i} \c_{\min} \\
&\geq \sqrt{k-i} \left(\x_{\min}-\|\w\|_2\right) \\
&= \sqrt{k-i} (\x_{\min}-\|\bar{\A}^{\dagger}\e\|_2) \\
&\geq \sqrt{k-i} (\x_{\min}-\sigma_{\max}(\bar{\A}^{\dagger})\|\e\|_2) \\
&=\sqrt{k-i} (\x_{\min}-\sigma_{\min}(\bar{\A})\epsilon_{\e}). 
\end{aligned}
\end{equation}
where we imposed $\x_{\min}>\sigma_{\min}(\bar{\A})\epsilon_{\e}$. Combine   \ref{eq:rhobound} and \ref{eq:normc} to reach
\begin{equation}
\rho(\r_i)\leq M_{i+1}+\frac{\epsilon_\e}{\sigma_{\min}^2(\bar{\A})(\x_{\min}-\sigma_{\min}(\bar{\A})\epsilon_{\e})}
\end{equation}
Therefore, since $M_{i+1}<1$ by assumption, condition \ref{eq:conthm1}\footnote{It should be noted that \ref{eq:conthm1} does not conflict with our restriction of $\x_{\min}>\sigma_{\min}(\bar{\A})\epsilon_{\e}$.}
is sufficient for $\rho(\r_i)<1$ whence OLS selects a true column in iteration $(i+1)\ts{st}$. This completes the proof.
\end{proof}
\textit{Remark 1:} Note that Theorem \oldref{thm:erc} can be tailored to obtain ERC-based conditions for OMP by replacing $M_{i+1}$ in \ref{eq:conthm1} with $M_{OMP}$. Here we compare \ref{eq:conthm1} with the result of Proposition 1 in \cite{cai2011orthogonal}. If 
\begin{equation} \label{eq:compare}
\sigma_{\min}^3(\bar{\A}) < 1\slash (1-M_{OMP}),
\end{equation}
Proposition 1 in \cite{cai2011orthogonal} requires a more restrictive condition than \ref{eq:conthm1} on the small element of $\bar{\x}$. However, for nearly all $\A$ that are of interest in compressed sensing applications, \ref{eq:compare} is satisfied with high probability as $M_{OMP}$ is inversely proportional to $k$. Therefore, Theorem \oldref{thm:erc} ameliorates the existing ERC-based results for OMP as well as providing new sufficient conditions for OLS.
\section{Exploiting randomness in measurements} \label{sec:rand}
When $\A$ is drawn at random according to ${\cal N}(0,1/n)$ or ${\cal B}(\frac{1}{2},\pm \frac{1}{\sqrt{n}})$, concentration of measure inequalities such as that of Lemma \oldref{lem:23} hold. In particular, singular values of $\H$ are with high probability concentrated around $1$. Therefore, one may exploit these properties to establish probabilistic and perhaps pragmatic guarantees for performance of a sparse reconstruction algorithm. Theorem \oldref{thm:nois} below, states that for these matrices, OLS is capable of exact recovery of sparse signals with high probability if elements of $\x$ are sufficiently larger than noise.
\begin{theorem}\label{thm:nois}
\textit{Suppose $\x$ is an arbitrary sparse vector with sparsity level $k$ in $\R^m$. Choose a random matrix $\A \in \R^{n\times m}$ such that its entries are drawn uniformly and independently from ${\cal N}(0,1/n)$ or ${\cal B}(\frac{1}{2},\pm \frac{1}{\sqrt{n}})$. Fix $0<\epsilon<1$, $0<\delta<1$, and $t >0$. Given the noisy measurements $\y=\A\x+\e$ where $\e$ is independent of $\A$ and $\x$ and $\|\e\|_2\leq \epsilon_{\e}$, if $\min_{\x_j \ne 0}{|\x_j|}\geq (1+\delta+t) \epsilon_{\e}$, OLS recovers $\x$ in $k$ iterations with probability of success exceeding
\begin{multline}\label{thm:rand}
\geq \left(1-2e^{-(n-k+1)c_0(\epsilon)}\right)^2\left(1-2(\frac{12}{\delta})^ke^{-nc_0(\frac{\delta}{2})}\right)\\
\left(1-2\sum_{i=0}^{k-1} e^{-\frac{n \frac{1-\epsilon}{1+\epsilon}(1-\delta)^4}{k\left[\frac{1}{(k-i) t^2}+(1+\delta)^2\right]}}\right)^{m-k}.
\end{multline}}
\end{theorem}
\begin{proof}
In proof of Theorem \oldref{thm:erc} we discussed that \ref{eq:gsratio} is a sufficient condition for OLS to select a true column at next iteration. Therefore, if $\Sigma$ denotes the event that OLS succeeds, then $\Pr\{\Sigma\}\geq \Pr\{\max_{i}{\rho(\r_i)}<1\}$. Recall the idempotent property for $\P_i^\bot$, i.e.,
\begin{equation}\label{eq:idemp} 
\P_i^\bot={\P_i^\bot}^2={\P_i^\bot}^\top
\end{equation}
Employing Lemma \oldref{lem:23}, along with \ref{eq:idemp} delivers
\begin{equation} \label{eq:thm11}
\rho(\r_i)\leq\frac{1}{c_1(\epsilon)}\frac{\|\widetilde{\A}^\top \r_i\|_\infty}{\|\bar{\A}^\top \r_i\|_\infty}
\end{equation}
with probability exceeding $p_1=\left(1-2e^{-(n-k+1)c_0(\epsilon)}\right)^2$ for $0\leq i<k$ where $c_1(\epsilon)=\sqrt{\frac{1-\epsilon}{1+\epsilon}}$. Following the framework in \cite{tropp2007signal}, a simple norm inequality and the fact that $\bar{\A}^\top \r_i$ has at most $k$ nonzero entries results in
\begin{equation}
\rho(\r_i)\leq \frac{\sqrt{k}}{c_1(\epsilon)}\|\widetilde{\A}^\top \widetilde{\r}_i\|_\infty
\end{equation}
where $\widetilde{\r}_i=\r_i\slash \|\bar{\A}^\top \r_i\|_2$. Consequently, We examine an upper bound for $\widetilde{\r}_i$.
Employ Lemma \oldref{lem:nois} and Lemma \oldref{lem:3} in definition of $\widetilde{\r}_i$ to reach
\begin{equation}\label{eq:normrt}
\|\widetilde{\r}_i\|_2 \leq \frac{\left[\|\e^\bot\|_2^2\slash \|\c_{i^c}\|_2^2+\sigma_{\max}^2(\bar{\A})\right]^{\frac{1}{2}}}{\sigma_{\min}^2(\bar{\A})}.
\end{equation}
It is shown in \cite{baraniuk2008simple} that for any $0<\delta<1$, 
\begin{equation} \label{eq:p2}
\Pr\{1-\delta\leq\sigma_{\min}(\bar{\A})\leq 1+\delta\}\geq 1-2(\frac{12}{\delta})^ke^{-nc_0(\frac{\delta}{2})}.
\end{equation}
Call the term on the right hand side of \ref{eq:p2} $p_2$.
Combining \ref{eq:normrt}, \ref{eq:p2}, \ref{eq:normc}, and the fact that $\|\e^\bot\|_2 \leq \|\e\|_2 \leq \epsilon_{\e}$ furnishes 
\begin{equation}
\|\widetilde{\r}_i\|_2\leq \frac{\left[\frac{\epsilon_{\e}^2}{(k-i) (\x_{\min}-(1+\delta)\epsilon_{\e})^2}+(1+\delta)^2\right]^{\frac{1}{2}}}{(1-\delta)^2}
\end{equation}
with probability exceeding $p_2$. Thus, imposing the constraint 
$\x_{\min}\geq (1+\delta+t)\epsilon_{\e}$
 for any $t>0$ 
 establishes
\begin{equation} \label{eq:rtilub}
\begin{aligned}
\|\widetilde{\r}_i\|_2 \leq \frac{\left[\frac{1}{(k-i) t^2}+(1+\delta)^2\right]^{\frac{1}{2}}}{(1-\delta)^2}.
\end{aligned}
\end{equation}
Using the independence assumption of columns of $\widetilde{\A}$, the fact that $\{\widetilde{\r}_i\}_{i=0}^{k-1}$ are bounded with probability higher than $p_2$ and are statistically independent of $\widetilde{\A}$, and applying Boole's inequality or Hoeffding inequality\footnote{Depending on whether $\A \sim {\cal N}\left(0,\frac{1}{n}\right)$ or $\A \sim {\cal B}(\frac{1}{2},\pm \frac{1}{\sqrt{n}})$.} we reach \ref{thm:rand} which completes the proof.
\end{proof}
\textit{Remark 2:} If we define $\SNR=\frac{\|\A\x\|_2^2}{\|\e\|_2^2}$, the condition 
$\min_{\x_j \ne 0}{|\x_j|}\geq (1+\delta+t) \epsilon_\e$ implies
\begin{equation}
\SNR \approx k (1+\delta+t)^2,
\end{equation} 
which  suggests that for exact support recovery via OLS, $\SNR$ should scale linearly with sparsity level.

\textit{Remark 3:} Note that when $k \to \infty$ (so do $m$, and $n$), $p_1$ , $p_2$, and $p_3$ overwhelmingly approach $1$. Therefore, for these set of parameters, one may assume very small $\epsilon$ and $\delta$.

With some numerical estimates on lower bound of success probability  which was established in Theorem \oldref{thm:nois}, one may establish a lower bound on sufficient number of measurements for exact support recovery. Specifically, if conditions within Theorem \oldref{thm:nois} are satisfied, for any $0<\gamma<1$, there exist positive constants $C_1$, $C_2$, and $C_3$ which are independent of $\gamma$, $n$, $m$, and $k$ such that OLS succeeds with $\Pr\{\Sigma\}\geq 1-\gamma^2$ if $n\geq \max\{\frac{2}{C_1} k \log \frac{m}{\gamma}, C_2 k+ \log \frac{12}{\gamma^2}\slash C_3\}$. Hence, OLS recovers $k$-sparse $\x \in \R^{m}$ if the number of measurements grow linearly in $k$ and logarithmically in $m$.
\section{CONCLUSION} \label{sec:concl}
In this paper, we established sufficient conditions for exact support recovery via Orthogonal Least-Squares (OLS) in noisy setting. In particular, if ERC for OLS holds, and SNR is adequately high, OLS recovers all true indices.  
We also showed that for Gaussian and Bernoulli coefficient matrices, OLS is with 
high probability guaranteed to reconstruct any sparse signal from a low number of noisy random linear measurements if its elements are to too small.
\bibliographystyle{IEEEbib}\small
\bibliography{IEEEabrv,refs}
\end{document}